\pdfoutput=1
\documentclass[11pt]{article}

\usepackage{acl}

\usepackage{times}
\usepackage{latexsym}
\usepackage{fancyhdr}

\usepackage[]{natbib}

\usepackage[T2A,T1]{fontenc}
\usepackage[utf8]{inputenc}
\usepackage[russian,english]{babel}

\usepackage{microtype}

\usepackage{times}
\usepackage{latexsym}

\usepackage{hyperref}
\usepackage{xcolor}
\usepackage{amsmath}
\usepackage{amssymb}
\usepackage{amsthm}

\usepackage{graphicx}
\usepackage{booktabs}
\usepackage{multirow}
\usepackage{verbatim}
\usepackage{tabularx}

\usepackage{qtree}
\usepackage{tikz}
\usepackage{tikz-qtree,tikz-qtree-compat}
\tikzset{every tree node/.style={align=center, anchor=north}}
\usetikzlibrary{positioning}
\usetikzlibrary{matrix}

\newtheorem{proposition}{Proposition}[section]

\DeclareMathOperator{\softmax}{softmax}

\title{From Hyperbolic Geometry Back to Word Embeddings}

\author{Sultan Nurmukhamedov \\
  Yandex School of Data Analysis \\
  \texttt{soltustik@gmail.com} \\\And
  Thomas Mach \\
  University of Potsdam \\
  \texttt{mach@uni-potsdam.de} \AND
  Arsen Sheverdin \\
  University of Amsterdam \\
  \texttt{arsen.sheverdin@student.uva.nl} \\\And
  Zhenisbek Assylbekov \\
  Nazarbayev University \\
  \texttt{zhassylbekov@nu.edu.kz}
  }

\date{}

\begin{document}
\maketitle
\thispagestyle{fancy}
\begin{abstract}

We choose random points in the hyperbolic disc and claim that these points are already word representations. However, it is yet to be uncovered which point corresponds to which word of the human language of interest. This correspondence can be approximately established using a pointwise mutual information between words and recent alignment techniques.
\end{abstract}

\section{Introduction}
Vector representations of words are ubiquitous in modern natural language processing (NLP). There are currently two large classes of word embedding models: they build (1) static and (2) contextualized word vectors correspondingly. 

Static embeddings map each \textit{word type} into a vector of real numbers, regardless of the context in which the word type is used. The most prominent representatives of this class of models are {\sc word2vec} \cite{mikolov2013distributed,mikolov2013efficient} and {\sc GloVe} \cite{pennington2014glove}. The obvious problem with this approach is the representation of polysemous words, such as \textit{bank}---it becomes unclear whether we are talking about a financial institution, or we are talking about the river bank. 

Contextualized word embeddings, such as {\sc ELMo} \cite{peters2018deep} and {\sc BERT} \cite{devlin2019bert}, solve this problem by mapping each \textit{word token} into a vector space depending on the context in which the given word token is used, i.e.\ the same word will have different vector representations when used in different contexts. The second
approach can nowadays be considered mainstream, despite relatively few papers offering theoretical justifications for contextualized word embeddings. 

For static embeddings, on the contrary, there is a number of theoretical works, each of which offers its own version of what is happening when word vectors are trained. An incomplete list of such works includes those of \newcite{levy2014neural}, \newcite{arora2016latent}, \newcite{hashimoto2016word}, \newcite{gittens2017skip}, \newcite{tian2017mechanism}, \newcite{ethayarajh2019towards}, \newcite{allen2019vec}, \newcite{allen2019analogies}, \newcite{assylbekov2019context},  \newcite{zobnin2019learning}. Other advantages of static
embeddings over contextualized ones include faster training (few hours instead
of few days) and lower computing requirements (1 consumer-level GPU instead of
8--16 non-consumer GPUs). Morevoer, static embeddings are still an integral part
of deep neural network models that produce contextualized word vectors, because
embedding lookup matrices are used at the input and output (softmax) layers of
such models. Therefore, we consider it necessary to further study static
embeddings.

Several recent works \cite{nickel2017poincare,tifrea2018poincar} argue that
static word embeddings should be better trained in hyperbolic spaces than in
Euclidean spaces, and provide empirical evidence that word embeddings trained in
hyperbolic spaces need less dimensions to achieve the same quality as
state-of-the-art Euclidean vectors.\footnote{The quality of word vectors is
  usually measured by the performance of downstream tasks, such as similarity,
  analogies, part-of-speech tagging, etc.} Usually such works motivate the
hyperbolicity of word embeddings by the fact that hyperbolic spaces are better
suited for embedding hierarchical structures. Words themselves often denote concepts
with an underlying hierarchy. An example of such a hierarchy is
the {\sc WordNet} database, an excerpt of which is shown in
Fig.~\ref{fig:wordnet}. 
\begin{figure}[htbp]
\centering
\begin{tikzpicture}
\Tree
[.\node{carnivore};
    [.\node{feline};
        [.\node{big cat}; 
            [.\node{lion}; ]
            [.\node{tiger}; ]
        ]
        [.\node{cat}; ]
    ]
    [.\node{canine};
        [.\node{dog}; ]
        [.\node{wolf}; ]
        [.\node{fox}; ]
    ] 
]
\end{tikzpicture}
\caption{An excerpt from the {\sc WordNet} database.}
\label{fig:wordnet}
\end{figure}
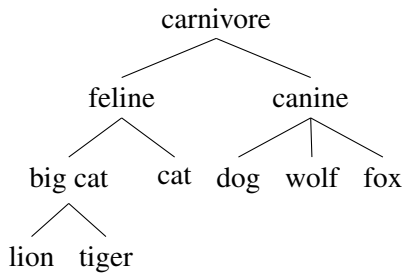
%We find this motivation unsatisfactory.

In the present paper we will investigate where the hyperbolicity originates
from.  If we take the state-of-the-art Euclidean embeddings, is it possible to
establish a direct connection between them and their counterparts from a
hyperbolic word embedding? This was answered positively by \citet{assylbekov2020binarized} who established a chain of
connections: from word embeddings to co-occurrence matrices, then to complex networks, and, finally, to hyperbolic spaces. In this paper, to provide an additional
justification for the constructed chain, we propose a way to move from the final
point, hyperbolic spaces, to the initial one, word embeddings. We show that drawing random points from the
hyperbolic plane results in a set of points that reasonably well resembles word
embeddings. In fact, we can match these points to word embeddings. Contrary, the
same trick does not work with points drawn at random in the Euclidean
space. Thus, one can argue that the hyperbolic space provides the underlying
structure for word embeddings, while in the Euclidean space this structure has
to be superimposed. 

\subsection*{Notation}
We denote with $\mathbb{R}$ the real numbers. Bold-faced lowercase letters ($\mathbf{x}$) denote vectors, plain-faced lowercase letters ($x$) denote scalars, bold-faced uppercase letters ($\mathbf{A}$) denote matrices, $\langle\mathbf{x},\mathbf{y}\rangle$ is the Euclidean inner product. We use $\mathbf{A}_{a:b,c:d}$ to denote a submatrix located at the intersection of rows $a, a+1, \ldots, b$ and columns $c, c + 1, \ldots, d$ of $\mathbf{A}$.
`i.i.d.' stands for `independent and identically distributed', `p.d.f' stands for `probability distribution function'. 
We use the sign $\propto$ to abbreviate `proportional to', and the sign $\sim$ to abbreviate `distributed as'. 

Assuming that words have already been converted into indices, let $\mathcal{W}:=\{1,\ldots,n\}$ be a finite vocabulary of words. Following the setup of the widely used {\sc word2vec} model \cite{mikolov2013efficient,mikolov2013distributed}, we use \textit{two} vectors per each word $i$: (1) $\mathbf{w}_i\in\mathbb{R}^d$ when $i\in\mathcal{W}$ is a center word, (2) $\mathbf{c}_i\in\mathbb{R}^d$ when $i\in\mathcal{W}$ is a context word; and we assume that $d\ll n$. %Word vectors $\{\mathbf{w}_i\}$ are also known as \textit{word embeddings}, while context vectors $\{\mathbf{c}_i\}$  are also known as \textit{context embeddings}.

In what follows we assume that our dataset consists of co-occurence pairs $(i,j)$. We say that ``the words $i$ and $j$ co-occur'' when they co-occur in a fixed-size window of words. %E.g., using a window of size 1 we can convert the text \textit{the cat sat on the mat} into a set of pairs: (\textit{the}, \textit{cat}), (\textit{cat}, \textit{the}), (\textit{cat}, \textit{sat}), (\textit{sat}, \textit{cat}), (\textit{sat}, \textit{on}), (\textit{on}, \textit{sat}), (\textit{on}, \textit{the}), (\textit{the}, \textit{on}), (\textit{the}, \textit{mat}), (\textit{mat}, \textit{the}). %The number of such pairs, i.e.\ the size of our dataset, is denoted by $N$. 
Let $\#(i,j)$ be the number of times the words $i$ and $j$ co-occur.%, then $N=\sum_{i\in\mathcal{W}}\sum_{j\in\mathcal{W}}\#(i,j)$.

\begin{figure*}
    \begin{minipage}[t]{.39\textwidth}
    \includegraphics[width=\textwidth]{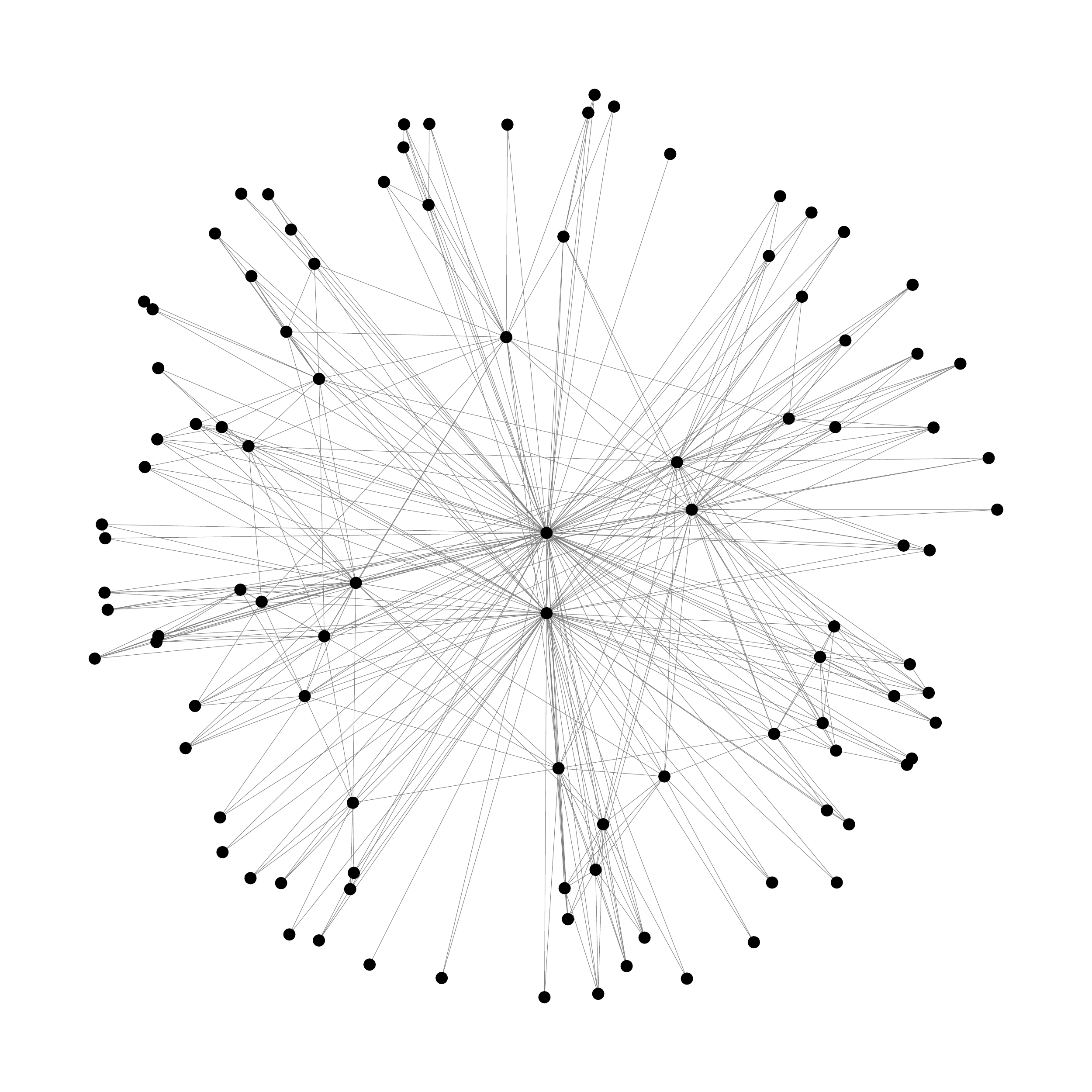}
    \caption{Random hyperbolic graph.}
    \label{fig:rhg}
    \end{minipage}\hfill\begin{minipage}[t]{.59\textwidth}\includegraphics[width=\textwidth]{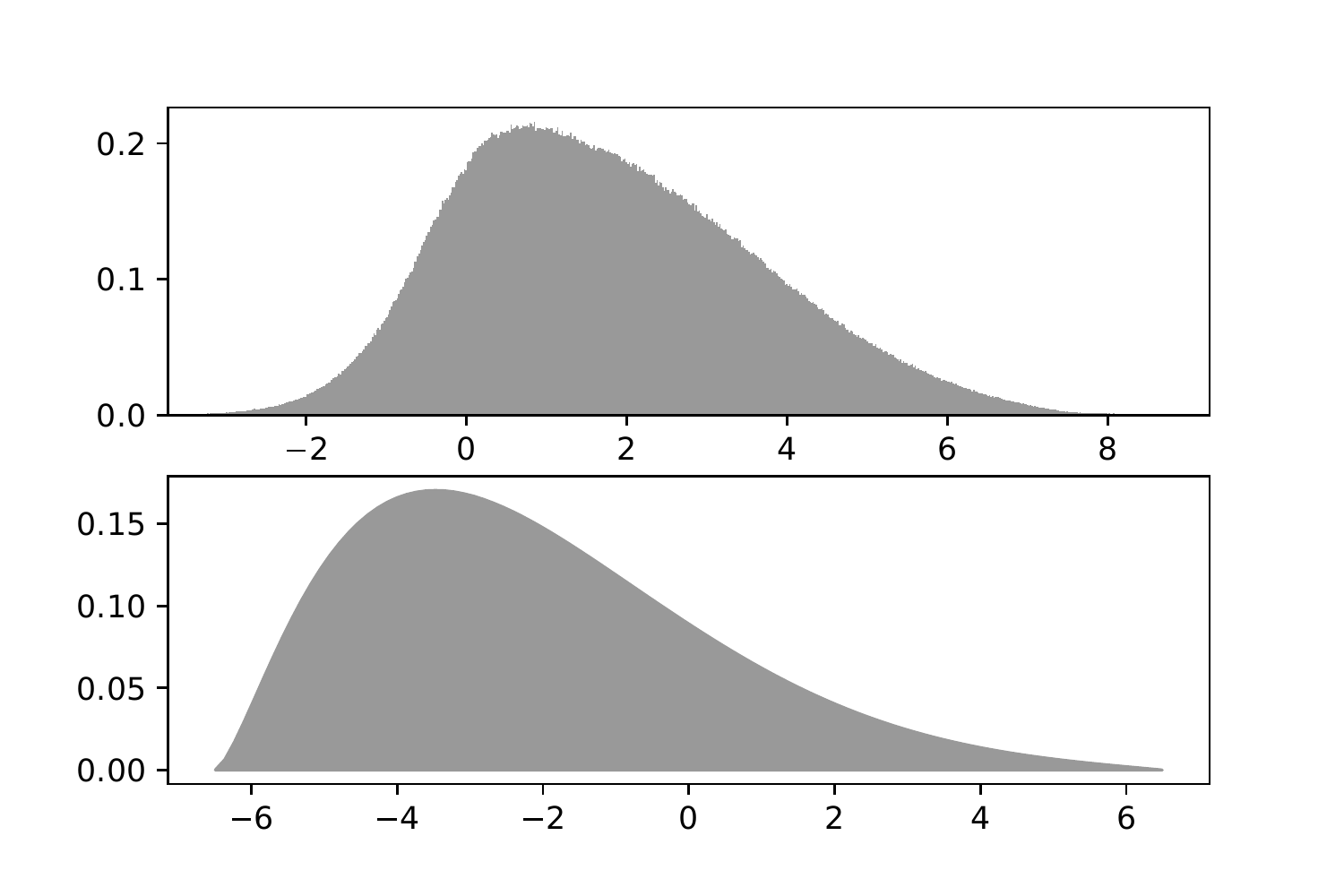}
    \caption{Distribution of PMI values (top) and of $R-X$.}
    \label{fig:pmi_and_dist}
    \end{minipage}
\end{figure*}

\section{Background: From Word Embeddings to Hyperbolic Space}
Our departure point is the skip-gram with negative sampling (SGNS) word embedding model of \newcite{mikolov2013distributed} that maximizes the following objective function
\begin{multline}
    \sum_{i\in\mathcal{W}}\sum_{j\in\mathcal{W}}\#(i,j)\log\sigma(\langle\mathbf{w}_i,\mathbf{c}_j\rangle)\\+k\cdot\mathbb{E}_{j'\sim p}[\log
    \sigma(-\langle\mathbf{w}_i,\mathbf{c}_{j'}\rangle)],\label{eq:sgns}
\end{multline}
where $\sigma(x)=\frac{1}{1+e^{-x}}$ is the logistic sigmoid function, $p$ is a smoothed unigram probability distribution for words,\footnote{The authors of SGNS suggest $p(i)\propto\#(i)^{3/4}$.} and $k$ is the number of negative samples to be drawn. Interestingly, training SGNS is approximately equivalent to finding a low-rank approximation of a shifted pointwise mutual information (PMI) matrix \cite{levy2014neural} in the form 
\begin{equation}
\log\frac{p(i,j)}{p(i)p(j)}-\log k\approx\langle\mathbf{w}_i,\mathbf{c}_j\rangle,\label{spmi}
\end{equation}
where the left-hand side is the 
shifted PMI between $i$ and $j$, and the right-hand side is an $ij$-th element of a matrix with rank $\le d$ since $\mathbf{w}_i,\mathbf{c}_j\in\mathbb{R}^d$. This approximation was later re-derived by \newcite{arora2016latent}, \newcite{zobnin2019learning}, \newcite{assylbekov2019context}, and \newcite{allen2019vec} under different sets of assuptions. In a recent paper, \newcite{assylbekov2020binarized} showed that the removal of the sigmoid transformation in the SGNS objective \eqref{eq:sgns} gives word embeddings comparable in quality with the original SGNS embeddings. A maximization of such modified objective results in a low-rank approximation of a \textit{squashed shifted} PMI ($\sigma$SPMI) matrix, defined as 
\begin{equation}
\mathbf{A}_{ij}:=\sigma\left(\log\frac{p(i,j)}{p(i)p(j)}-\log k\right).\label{eq:bpmi}
\end{equation}
Moreover, treating the $\sigma$SPMI matrix as a connection probabilities matrix of a random graph, the authors show that such graph is a \textit{complex network}, that is it has strong clustering and scale-free degree distribution, and according to \newcite{krioukov2010hyperbolic}, such graph possesses an effective hyperbolic geometry underneath. The following chain summarizes this argument:
\begin{multline}
\boxed{\text{Word Embeddings}}\quad\longrightarrow\quad\boxed{\sigma\text{SPMI}}\quad\longrightarrow\\\boxed{\text{Complex Network}}\quad\longrightarrow\quad\boxed{\text{Hyperbolic Space}}\notag
\end{multline}
In our work, we go from the final point (hyperbolic space) to the starting one (word embeddings), and the next section provides the details of our method.

\begin{table*}[htbp]
\begin{center}
\begin{tabular}{l | c c c | c c}
\toprule
% & \multicolumn{3}{c}{WS353} \\
 \multirow{2}{*}{Method} & \multicolumn{3}{c|}{Word Similarity} & \multicolumn{2}{c}{POS Tagging}\\
 & {\sc WS353} & {\sc Men} & {\sc M. Turk} & {\sc CoNLL-2000} & {\sc Brown}\\
\midrule
 SGNS                & .678 & .656 & .690 & 90.77 & 92.60\\
 PMI + SVD           & .669 & .674 & .666 & 92.25 & 93.76\\
 $\sigma$SPMI + SVD  & .648 & .622 & .666 & 92.76 & 93.78\\
 RHG + SVD + Align   & .406 & .399 & .509 & 92.23 & 93.19\\
 Random + Align      & .165 & .117 & .111 & 81.89 & 89.39\\
\bottomrule
\end{tabular}
\end{center}
\caption{Evaluation of word embeddings on the similarity and POS tagging tasks. For the similarity tasks the evaluation metric is the Spearman's correlation with human ratings, for the POS tagging tasks it is accuracy. \textit{Random} stands for random vectors that were obtained as i.i.d. draws from $\mathcal{N}(\mathbf{0},\mathbf{I})$.}
\label{tab:emb_eval}
\end{table*}

\section{Method: From Hyperbolic Geometry to Word Embeddings}

It is difficult to visualize hyperbolic spaces because they cannot be isometrically embedded into any Euclidean space.\footnote{This means that we cannot map points of a hyperbolic space into points of a Euclidean space in such way that the distances between points are preserved.} However, there exist  models of hyperbolic spaces: each model emphasizes different aspects of hyperbolic geometry, but no model simultaneously represents all of its properties. We will consider here the so-called \textit{native} model  \cite{krioukov2010hyperbolic}, in which the hyperbolic plane $\mathbb{H}^2$ is represented by a disk of radius $R$, and we use polar coordinates $(r,\theta)$ to specify the position of any point $v\in\mathbb{H}^2$, where the radial coordinate $r$ equals the hyperbolic distance of $v$ from the origin. Given this notation, the distance $x$ between two points with coordinates $(r,\theta)$ and $(r',\theta')$ satisfies the hyperbolic law of cosines
\begin{multline}
\textstyle\cosh x=\cosh r\cosh r'\\-\sinh r\sinh r'\cos(\theta-\theta'),\label{eq:hyper_dist}
\end{multline}
for the hyperbolic space of constant curvature $-1$.\footnote{Defining constant curvature is beyond the scope of our paper. We just mention here that there are only three types of isotropic spaces: Euclidean
(zero curvature), spherical (positively curved), and hyperbolic (negatively curved).}
A key property of hyperbolic spaces is that they expand faster than Euclidean spaces. E.g., a circle with radius $r$ has in the Euclidean plane a length of $2\pi r=\Theta(r)$ and an area of $\pi r^2=\Theta(r^2)$, while its length and area in the hyperbolic plane are $2\pi\sinh( r)=\Theta(e^{r})$ and $2\pi(\cosh r-1)=\Theta(e^{r})$ correspondingly. %\footnote{$f(x)=\Theta(g(x))$ means that there exist $x_0$ and  constants $c_1$ and $c_2$ such that $c_1 g(x)\le f(x)\le c_2 f(x)$ for $x\ge x_0$.} 
It is noteworthy that in a balanced tree with branching factor $b$, the number
of nodes that are $r$ edges from the root grows as $\Theta(b^r)$, i.e.\
exponentially with $r$, leading to the suggestion
that hierarchical complex networks with tree-like structures
might be easily embeddable in hyperbolic space.

Based on the above facts, we construct a random hyperbolic (RHG) graph as in the work of \newcite{krioukov2010hyperbolic}: we place randomly $n$ points (nodes) into a hyperbolic disk of radius $R$, and each pair of nodes $(i,j)$ is connected with probability $\sigma(R-x_{ij})$, where $x_{ij}$ is the hyperbolic distance \eqref{eq:hyper_dist} between points $i$ and $j$. Angular coordinates of the nodes are sampled from the uniform distribution: $\theta\sim\mathcal{U}[0,2\pi]$, while the radial coordinates are sampled from the exponential p.d.f. 
$$
\rho(r)=\frac{\alpha\sinh\alpha r}{\cosh \alpha R-1}=\Theta(e^{\alpha r}).
$$
The hyperparameters $R$ and $\alpha$ are chosen based on the total number of nodes $n$, the desired average degree $\bar{k}$ and the power-law exponent $\gamma$ according to the equations (22) and (29) of \newcite{krioukov2010hyperbolic}. An example of such RHG is shown in Figure~\ref{fig:rhg}. Notice, that the connection probabilities matrix of our graph is 
$$
\mathbf{B}_{ij}:=\sigma(R-x_{ij}),
$$
Comparing this to \eqref{eq:bpmi}, we see that if $\mathbf{A}$ and $\mathbf{B}$ induce structurally similar graphs then the distribution of the PMI values $\log\frac{p(i,j)}{p(i)p(j)}$ should be similar to the distribution of $R-x_{ij}$ values (up to a constant shift). To test this empirically, we compute a PMI matrix of a well-known corpus, \texttt{text8}, %\footnote{\url{http://mattmahoney.net/dc/textdata.html}} 
and compare the distribution of the PMI values with the p.d.f.\ of $R-X$, where $X$ is a distance between two random points of a hyperbolic disk (the exact form of this p.d.f.\ is given in Proposition~\ref{prop:dist_distr}). The results are shown in Figure~\ref{fig:pmi_and_dist}. As we can see, the two distributions are similar in the sense that both are unimodal and right-skewed. The main difference is in the shift---distribution of $R-X$ is shifted to the left compared to the distribution of the PMI values. %This possibly explains why the PMI entries are shifted to the \textit{left} in the matrix factorization approach of \newcite{levy2014neural}. 

We hypothesize that the nodes of the RHG treated as points of the hyperbolic space are \textit{already} reasonable word embeddings for the words of our vocabulary $\mathcal{W}$. The only thing that we do not know is the correspondence between words $i\in\mathcal{W}$ and nodes of the RHG. Instead of aligning words with nodes, we can align their vector representations. For this, we take singular value decompositions (SVD) of $\mathbf{A}$ and $\mathbf{B}$:
$$
\mathbf{A}=\mathbf{U}_A\boldsymbol\Sigma_A\mathbf{V}_A^\top,\quad\mathbf{B}=\mathbf{U}_B\boldsymbol\Sigma_B\mathbf{V}_B^\top,$$ 
and then obtain embedding matrices by
\begin{align*}
\mathbf{W}_A&:=\mathbf{U}_{A,1:n,1:d}\boldsymbol\Sigma^{1/2}_{A,1:d,1:d}\in\mathbb{R}^{n\times d}\\\mathbf{W}_B&:=\mathbf{U}_{B,1:n,1:d}\boldsymbol\Sigma^{1/2}_{B,1:d,1:d}\in\mathbb{R}^{n\times d}
\end{align*}
as in the work of \citet{levy2014neural}. An $i^\text{th}$ row in $\mathbf{W}_A$ is an embedding of the word $i\in\mathcal{W}$, while an $i^{\text{th}}$ row in $\mathbf{W}_B$ is an embedding of the RHG's node $i$. To align these two sets of embeddings we apply a recent stochastic optimization method of \citet{grave2019unsupervised} that solves
$$
\min_{\mathbf{Q}\in\mathcal{O}_d}\min_{\mathbf{P}\in\mathcal{P}_n}\|\mathbf{W}_A\mathbf{Q}-\mathbf{P}\mathbf{W}_B\|^2_2,
$$
where $\mathcal{O}_d$ is the set of $d\times d$ orthogonal matrices and $\mathcal{P}_d$ is the set of $n\times n$ permutation matrices. As one can see, this method assumes that \textit{alignment} between two sets of embeddings is not only a permutation from one set to the other, but also an orthogonal transformation between the two. Once the alignment is done, we treat $\mathbf{PW}_B$ as an embedding matrix for the words in $\mathcal{W}$.

\section{Evaluation}

In this section we evaluate the quality of word vectors resulting from a RHG\footnote{Our code is available at \url{https://github.com/soltustik/RHG}} against those from the SGNS, PMI, and $\sigma$SPMI. We use the \texttt{text8} corpus mentioned in the previous section.
We were ignoring words that appeared less than 5 times (resulting in a vocabulary of 71,290 tokens). We set window size to 2, subsampling threshold to $10^{-5}$, and dimensionality of word vectors to 200. The SGNS embeddings were trained using our custom implementation.\footnote{\url{https://github.com/zh3nis/SGNS}} The PMI and BPMI matrices were extracted using the {\sc hyperwords} tool of \newcite{levy2015improving} and SVD was performed using the {\sc PyTorch} library of \newcite{paszke2019pytorch}.

The embeddings were evaluated on word similarity and POS tagging tasks. For word similarity we used {\sc WordSim} \citep{finkelstein2002placing}, {\sc MEN} \citep{bruni2012distributional}, and {\sc M.Turk} \citep{radinsky2011word} datasets. For POS tagging we trained a simple classifier\footnote{feedforward neural network with one hidden layer and softmax output layer} by feeding in the embedding of a current word and its nearby context to predict its part-of-speech (POS) tag:
$$
\widehat{\mathrm{POS}}_t = \softmax(\sigma(\mathbf{A}[\mathbf{w}_{t-2};\ldots;\mathbf{w}_{t+2}]+\mathbf{b}))
$$
where $[\mathbf{x};\mathbf{y}]$ is concatenation of $\mathbf{x}$ and $\mathbf{y}$. The classifier was trained on {\sc CoNLL-2000} \cite{tjong-kim-sang-buchholz-2000-introduction} and {\sc Brown} \cite{kucera1967computational} datasets.  

The results of evaluation are provided in Table~\ref{tab:emb_eval}. As we see, vector representations of words generated from a RHG lag behind in word similarity tasks from word vectors obtained by other standard methods. Note, however, that the similarity task was designed with Euclidean geometry in mind. Even though our RHG-based vectors are also ultimately placed in the Euclidean space (otherwise the alignment step would not have been possible), their nature is inherently non-Euclidean. Therefore, the similarity scores for them may not be indicative. So, for example, when RHG vectors are fed into a nonlinear model for POS tagging, they are comparable with other types of vectors. 

We notice that random vectors---generated as i.i.d. draws from $\mathcal{N}(\mathbf{0},\mathbf{I})$ and then aligned to the embeddings from $\sigma$SPMI---show poor results in the similarity tasks and underperform all other word embedding methods in the POS tagging tasks. This calls into question whether multivariate Gaussian is a reasonable (prior) distribution for word vectors as was suggested by \citet{arora2016latent}, \citet{assylbekov2019context}.

%Our goal was not to beat state-of-the-art, but rather to find a simple mathematical structure for word vectors, and hyperbolic space seems to be one of the possibilities. In the future, we consider it necessary to study hyperbolic spaces of higher dimension \cite{tifrea2018poincar}, as well as products of Euclidean, spherical and hyperbolic spaces \cite{gu2018learning}.

\section{Conclusion and Future Work}
  %We have shown that by throwing points randomly on the hyperbolic plane, we get word representations such that each point corresponds to a certain word of the human language, and this correspondence is determined by the relation (hyperbolic distance) to other words. 
  
  %Claiming that hyperbolicity underlies word vectors is not novel---\citet{nickel2017poincare}, \citet{tifrea2018poincar} successfully trained word vectors in hyperbolic spaces. However, in this work we show that word vectors can be obtained from hyperbolic geometry without explicit training---we can get the embeddings by throwing points randomly on the hyperbolic plane and then finding correspondence between the points and words of the human language. This correspondence is determined by the relation (hyperbolic distance) to other words. 
  
  %This conclusion is fully consistent with the principle of semiotic arbitrariness of \newcite{de2011course}---the relationship between a word (sign) and the real-world thing it denotes is an arbitrary one, i.e. things are what they are, no matter what name you give them.
  
  In this work we show that word vectors can be obtained from hyperbolic geometry
without explicit training. We obtain the embeddings by randomly drawing points
in the hyperbolic plane and by finding correspondence between these points and
the words of the human language. This correspondence is determined by the
relation (hyperbolic distance) to other words. This method avoids the, often
expensive, training of word vectors in hyperbolic spaces as in
\citet{tifrea2018poincar}. A direct comparison is not what this paper
attempts---our method is cheaper but produces word vectors of lower quality. Our
method simply shows that word vectors do fit better into hyperbolic space than
into Euclidean space. %We regard this as the main contribution of this paper.

Finally, we want to sketch a possible direction for future work. The hyperbolic
space is a special case of a Riemannian manifold. Are Riemannian manifolds
better suited for word vectors? In particular which manifolds should one use?
At the moment, there is only limited empirical knowledge to address these
questions. For instance, \citet{DBLP:conf/iclr/GuSGR19} obtained word
vectors of better quality, according to the similarity score, in the product of
hyperbolic spaces, which is still a Riemannian manifold but not a hyperbolic
space anymore. We are hopeful that future work may provide an explanation for
this empirical fact.

\section*{Acknowledgements}
Zhenisbek Assylbekov was supported by the Program of Targeted Funding ``Economy of the Future'' \#0054/\foreignlanguage{russian}{ПЦФ-НС}-19. The work of Sultan Nurmukhamedov was supported by the Nazarbayev University Faculty-Development Competitive Research Grants Program, grant number 240919FD3921. The authors would like to thank anonymous reviewers for their feedback.

\bibliography{ref}

\appendix

\section{Auxiliary Results} \label{app:dist}
\begin{proposition} \label{prop:dist_distr} Let $X$ be a distance between two points that were randomly uniformly placed in the hyperbolic disk of radius $R$. The probability distribution function of $X$ is given by
\begin{multline}
    f_X(x)=\int_0^R\int_0^R\\
    \frac{\sinh(x)\rho(r_1)\rho(r_2)dr_1dr_2}{\pi\sqrt{1-A(r_1,r_2,x)}\sinh(r_1)\sinh(r_2)},\label{eq:pdf_x}
\end{multline}
where $A(r_1,r_2,x)=\frac{\cosh(r_1)\cosh(r_2)-\cosh(x)}{\sinh(r_1)\sinh(r_2)}$, and $\rho(r)=\frac{\alpha\sinh\alpha r}{\cosh\alpha R-1}$.
\end{proposition}
\begin{proof}
Let us throw randomly and uniformly two points $(r_1,\theta_1)$ and $(r_2,\theta_2)$ into the hyperbolic disk of radius $R$, i.e.\ $r_1, r_2\,\,{\stackrel{\text{i.i.d.}}{\sim}}\,\,\rho(r)$, $\theta_1, \theta_2\,\,{\stackrel{\text{i.i.d.}}{\sim}}\,\,\text{Uniform}[0,2\pi)$. Let $X$ be the distance between these points ($X$ is a random variable). Let $\gamma$ be the angle between these points, then $\gamma:=\pi-|\pi-|\theta_1-\theta_2||\sim\text{Uniform}[0,\pi)$ and thus
$$
f_{\cos\gamma}(t)=\frac{1}{\pi\sqrt{1-t^2}},\quad t\in[-1,1].
$$
Since the distance in our model of hyperbolic plane is given by
$$
X=\cosh^{-1}[\cosh r_1\cosh r_2-\sinh r_1\sinh r_2\cos\gamma]
$$
we have
\begin{align*}
&\Pr(X\le x)\\
&=\Pr\left(\cos\gamma\ge\underbrace{\frac{\cosh r_1 \cosh r_2-\cosh x}{\sinh r_1\sinh r_2}}_{A(r_1,r_2,x)}\right)\\
&=\Pr(\cos\gamma\ge A(r_1,r_2,x))\\
&=\int_{A(r_1,r_2,x)}^{+\infty}\frac{1}{\pi\sqrt{1-t^2}}\\
&=\frac12-\frac{\sin^{-1}A(r_1,r_2,x)}{\pi},
\end{align*}
and therefore
\begin{multline*}
f_{X\mid r_1,r_2}(x)=\frac{d}{dx}\left[\frac12-\frac{\sin^{-1}A(r_1,r_2,x)}{\pi}\right]\\=\frac{\sinh x}{\pi\sqrt{1-A(r_1,r_2,x)}\sinh( r_1)\sinh r_2}
\end{multline*}
for $x\in(|r_1-r_2|,r_1+r_2)$.
Integrating $f_{X\mid r_1,r_2}(x)\rho(r_1)\rho(r_2)$  with respect to $r_1$ and $r_2$ we get \eqref{eq:pdf_x}.
\end{proof}

\end{document}